\documentclass[11pt,a4paper]{article}
\usepackage[utf8]{inputenc}
\usepackage{amsmath}
\usepackage{amsfonts}
\usepackage{amssymb}
\usepackage{amsthm}
\usepackage{mathtools}
\usepackage{subcaption}
\usepackage{paralist}
\usepackage{xparse}
\usepackage{authblk}
\usepackage[hyphens]{url}
\usepackage[pagebackref=True]{hyperref}
\usepackage[capitalize]{cleveref}
\usepackage{xspace}
\usepackage{tabularx,booktabs,multirow}
\usepackage{booktabs}

\usepackage{fullpage}

\usepackage{doi}
\usepackage[sort,numbers]{natbib}
\usepackage[backgroundcolor=gray!10]{todonotes}

\newcommand{\NN}{\mathbb{N}}
\newcommand{\Q}{\mathbb{Q}}
\newcommand{\eps}{\Delta}
\newcommand{\bigO}{\mathcal{O}}

\newcommand{\Gg}{\mathcal{G}}
\newcommand{\Hh}{\mathcal{H}}
\newcommand{\true}{\texttt{true}}
\newcommand{\false}{\texttt{false}}

\newcommand{\NP}{\textnormal{\textsf{NP}}\xspace}
\newcommand{\XP}{\textnormal{\textsf{XP}}\xspace}

\newcommand{\threesat}{\textnormal{\textsc{$3$-SAT}}\xspace}

\DeclarePairedDelimiterX{\set}[1]{\{ }{ \} }{\setargs{#1}}
\NewDocumentCommand{\setargs}{>{\SplitArgument{1}{;}}m}
{\setargsaux#1}
\NewDocumentCommand{\setargsaux}{mm}
{\IfNoValueTF{#2}{#1} {#1\::\:#2}}

\DeclarePairedDelimiterX{\abs}[1]{\lvert}{\rvert}{#1}

\DeclareMathOperator{\dtgw}{dtgw}

\newcommand{\dist}{\ensuremath{d}}
\newcommand{\DTGW}{\textsc{DTGW}\xspace}
\newcommand{\AM}{\textsc{AM}\xspace}

\newtheorem{theorem}{Theorem}[section]
\newtheorem{observation}[theorem]{Observation}
\newtheorem{corollary}[theorem]{Corollary}

\newcommand{\problemdef}[3]{
	\begin{center}
  \begin{minipage}{0.95\textwidth}
    \noindent
    \textsc{#1}

			\vspace{2pt}
			\setlength{\tabcolsep}{3pt}
			\begin{tabularx}{\textwidth}{@{}lX@{}}
					\textbf{Input:} 		& #2 \\
					\textbf{Question:} 	& #3
				\end{tabularx}
  \end{minipage}
	\end{center}
      }

\title{Comparing Temporal Graphs Using Dynamic Time Warping}

\author[1]{Vincent Froese}
\author[2]{Brijnesh Jain\footnote{Supported by the DFG project JA~2109/4-2.}}
\author[1]{Rolf Niedermeier}
\author[1]{Malte Renken\footnote{Supported by the DFG project NI~369/17-1.}}

\affil[1]{\small
  Algorithmics and Computational Complexity, Faculty~IV, TU Berlin, Berlin, Germany,\protect\\
  \{vincent.froese, rolf.niedermeier, m.renken\}@tu-berlin.de}
\affil[2]{\small Distributed Artificial Intelligence Laboratory, Faculty~IV, TU Berlin, Berlin, Germany\protect\\
brijnesh.jain@dai-labor.de}

\date{\today}

\begin{document}

\maketitle

\begin{abstract} 
  Within many real-world networks the links between pairs of nodes  change over time.
  Thus, there has been a recent boom in studying temporal graphs.
  Recognizing patterns in temporal graphs requires a proximity measure to compare different temporal graphs.
  To this end, we propose to study dynamic time warping on temporal graphs.
  We define the dynamic temporal graph warping distance (dtgw) to determine the dissimilarity of two temporal graphs.
  Our novel measure is flexible and can be applied in various application domains.
  We show that computing the dtgw-distance is a challenging (in general) \NP{}-hard optimization problem and identify some polynomial-time solvable special cases.
  Moreover, we develop a quadratic programming formulation and an efficient heuristic.
  In experiments on real-word data we show that the heuristic performs very well and that our dtgw-distance performs favorably in de-anonymizing networks compared to other approaches.
  
\medskip

\end{abstract}

\section{Introduction}
A fundamental concept for pattern recognition is the notion of proximity (i.e., (dis)similarity) between objects. 
For objects that are represented by numerical feature vectors, there exist a lot of well-known
proximity measures such as $p$-norms or positive semi-definite kernels.
In structural pattern recognition, objects are often more naturally represented by complex (discrete) data structures such as graphs, strings, or time series.
For these representations, one can often not simply use vector-based proximity measures.
Instead, one needs to define suitable domain-specific proximity measures such as the \emph{edit distance} on graphs or strings or the \emph{dynamic time warping} distance%
\footnote{Note that a \emph{distance} function is not required to obey the triangle inequality, in contrast to a \emph{metric}.}
 on time series.

The majority of graph proximity measures focuses on static graphs. This includes the graph edit distance~\cite{Riesen2015}, graph kernels \cite{KriegeJM20}, and geometric graph distances~\cite{Jain16}.
However, many complex systems are not static as the links between entities dynamically change over time.
Thus, there is a steadily growing research interest in analyzing \emph{temporal graphs} (we also use the term \emph{temporal network} interchangeably)~\cite{RG19,HS13,HS19}.
Such temporal graphs can be represented by a series of temporal edges between a fixed set of vertices.
Examples are social contact networks, disease spreading networks, traffic networks, attack networks in computer security, or protein-protein-interaction networks in biology~\cite{Kostakos09,HS13,HS19,LCLWB17,VCM17}.
Examples of data mining problems on temporal social networks include community detection \cite{DBSB19}, epidemics analysis~\cite{RGPV16}, and influence spreading~\cite{GPT15}.

Many processes described by temporal graphs naturally vary in duration and temporal dynamics (for example, chemical reactions or the spread of a disease might proceed at different speeds), which makes data mining tasks such as classification challenging.
Hence, one needs to find suitable proximity measures, which has seemingly not been done so far.

Our paper proposes such a measure.
We introduce a novel proximity measure on temporal graphs based on vertex signature graph distance and dynamic time warping, called \emph{dynamic temporal graph warping} (dtgw).
Dynamic time warping allows to cope with variations in temporal dynamics.
Thus, by combining established methods from graph-based pattern recognition and time series data mining in a nontrivial way, we obtain a suitable tool to analyze temporal network data.
We study the computational complexity of the dtgw-distance, develop efficient algorithms and study their behavior on real-world data, the latter indicating the strong potential for future applications.

\paragraph*{Related Work.}
Graph distance based on vertex mappings using local vertex signatures was introduced by~\citet{JT09}.
The idea of using vertex mappings can also be found in \emph{optimal assignment kernels}~\cite{FWSZ05,KGW16,BI18}.

Regarding proximity measures on temporal graphs, seemingly little work has been done so far. In fact, we are not aware of other approaches for numerically measuring the proximity of two temporal graphs.
Related concepts, however, have been investigated for temporal graphs.
For example, one approach is based on network embeddings where nodes are mapped into certain feature spaces incorporating the temporal behavior~\cite{AK15,ZLLGHW18,NLRAKK18}.
Another approach is based on network alignments~\cite{VCM17,ESCBK19} where a vertex mapping that optimizes some criteria is computed.
However, dynamic time warping has not been used in this context so far.

Dynamic time warping~\cite{SC78} is an established measure for mining time series data~\cite{RCMBWZZK12,WMDTSK13} which is specifically designed to cope with temporal distortion in the data via nonlinear alignment of time series.
It can be applied to time series of different lengths (in contrast to the Euclidean distance for example) which is a relevant aspect in time series averaging.
We lift this approved concept to the domain of temporal graphs.

\paragraph*{Our Contributions.}
We define the dynamic temporal graph warping (dtgw) distance as a twofold discrete minimization problem involving computation of an optimal vertex mapping and an optimal ``warping path'' (see \Cref{sec:DTGW}).
As a byproduct, our approach does not only yield a distance measure but also yields an interpretable mapping between vertices of the two temporal graphs which can, for example, be used for de-anonymization of individuals in social networks.

We show that the dtgw-distance is \NP-hard to compute in general~(\Cref{thm:nphard}). In contrast, we point out several polynomial-time solvable special cases.
This includes the case when either a vertex mapping or a temporal alignment is fixed (\Cref{obs:trivial}), the case of deciding whether the dtgw-distance is zero (\Cref{thm:c=0}), and the case when
the lifetimes of the two temporal graphs differ only by a constant and the warping path length is restricted (\Cref{prop:XPwarplength}).
Moreover, we give a quadratic programming formulation (\Cref{sec:qp}) and propose an efficient and effective heuristic approach (\Cref{sec:heuristic}).

We empirically evaluate the heuristic\footnote{An implementation is freely available at \url{www.akt.tu-berlin.de/menue/software}.} in preliminary experiments on real-world temporal social networks (face-to-face contact and spatial proximity networks) against the quadratic program and some simple baseline methods to show its efficiency and solution quality.
Moreover, we demonstrate that our concept can successfully be used for de-anonymization of real-world temporal social networks and is faster than other existing methods such as DynaMAGNA++~\cite{VCM17} or HTNE~\cite{ZLLGHW18} (\Cref{sec:experiments}). 

Compared to the conference version~\cite{FJNR19}, this version contains the NP-hardness proof of \Cref{thm:nphard}, the proof of~\Cref{prop:XPwarplength} and the quadratic programming formulation (\Cref{sec:qp}). Moreover, we present additional experimental results including a comparison of the heuristic with the quadratic program (\Cref{sec:benchmark}) as well as a clustering experiment (\Cref{sec:cluster}).

\paragraph*{Organization.}
\Cref{sec:prelim} contains basic definitions.
\Cref{sec:DTGW} presents our main definition of the dtgw-distance followed by NP-hardness results in \Cref{sec:hardness} and positive algorithmic results in \Cref{sec:algorithms}. \Cref{sec:experiments} presents experimental results on some real-world data.
We conclude in \Cref{sec:conclusion} with an outlook on future applications and challenges.

\section{Preliminaries}\label{sec:prelim}
For~$T\in\NN$, we define $[T] \coloneqq \{1, 2, \ldots, T\}$.
For a set~$S$, we denote the set of all size-$k$ subsets of~$S$ by~$\binom{S}{k}$.

\paragraph*{Temporal Graphs.}
A \emph{temporal graph} $\mathcal{G} = (V,E_1, E_2, \ldots, E_T)$ consists of a vertex set~$V$ and a sequence of~$T\ge 1$ edge sets $E_i \subseteq \binom{V}{2}$.
By $G_i = (V, E_i)$, we denote the \emph{$i$\textsuperscript{th} layer} of $\mathcal{G}$ and we call~$T$ the \emph{lifetime} of~$\Gg$.
The \emph{underlying graph} of~$\Gg$ is the graph~$(V,\bigcup_{i=1}^TE_i)$.
We remark that all definitions and results in this work can easily be extended to labeled temporal graphs (with vertex and/or edge labels).

\paragraph*{Vertex Mapping.}
A \emph{vertex mapping} between two vertex sets $V$ and $W$ is a set $M \subseteq V  \times W$ containing $\min(|V|,|W|)$ tuples such that each $x \in V \cup W$ is contained in at most one tuple of~$M$.
We denote the set of all vertex mappings between~$V$ and~$W$ by~$\mathcal{M}(V,W)$.
Let~$V_M\subseteq V$ be the subset of vertices in~$V$ that are contained in some tuple of~$M$
($W_M \subseteq W$ is defined analogously). Note that $V_M=V$ or $W_M=W$ holds since $|M|=\min(|V|,|W|)$.

\paragraph*{\textsc{Assignment Problem}.}
Computing optimal vertex mappings between two temporal graphs can be solved via 
the \textsc{Assignment Problem} which is a fundamental problem in combinatorial optimization. Given two sets~$A$ and~$B$ of equal size and a cost function~$c\colon A\times B \rightarrow \Q$, the goal is to find a bijection~$\pi \colon A \rightarrow B$ such that
$\sum_{a\in A}c(a,\pi(a))$
is minimized.
It is well known that the \textsc{Assignment Problem} is solvable in~$\bigO(|A|^3)$ time~\cite[Theorem~12.2]{AMO93}.

\paragraph*{Dynamic Time Warping.}
The dynamic time warping distance~\cite{SC78} is a distance between time series.
It is based on the concept of a warping path.
A \emph{warping path of order~$n\times m$} is a set $p = \set{p_1, \ldots, p_L}$ of~$L\ge 1$ pairs~$p_\ell=(i_\ell, j_\ell)$ such that
\begin{compactitem}
\item $p_1 = (1, 1)$ and $p_L = (n, m)$, and
\item $p_{\ell+1} \in \{(i_\ell+1, j_\ell+1), (i_\ell, j_\ell+1), (i_\ell+1, j_\ell)\}$ for all~$1\le\ell < L$.
\end{compactitem}
We denote the set of all warping paths of order~$n\times m$ by~$\mathcal{P}_{n,m}$.
For two temporal graphs~$\mathcal{G}=(V,E_{1},\ldots,E_T)$, $\mathcal{H}=(W,F_{1},\ldots,F_U)$, every order-$(T\times U)$ warping path~$p$ defines a \emph{warping} between $\mathcal{G}$ and $\mathcal{H}$,
that is, a pair~$(i,j)\in p$ \emph{warps} layer~$G_{i}$ to layer~$H_{j}$.

\paragraph*{Parameterized Complexity.}
We assume the reader to be familiar with basic concepts of computational complexity theory such as \NP-completeness \cite{GJ79}. In parameterized complexity theory~\cite{DF13,CFK+15} one considers running times with respect to two dimensions. One dimension is the size of the input instance~$x$ and the other dimension is a \emph{parameter}~$k$ (usually a numerical value).
An instance of a parameterized problem is a pair~$(x,k)$.
The class \XP contains all parameterized problems that can be solved in polynomial time for every constant parameter value, that is, in $|x|^{f(k)}$~time for some function~$f$ only depending on~$k$.%

\section{Dynamic Temporal Graph Warping (DTGW)}\label{sec:DTGW}

In this section, we define our temporal graph distance based on dynamic time warping using a vertex-signature-based graph distance as cost function.
We choose this graph distance for the following reasons.
First, in contrast to the NP-hard edit distance, it is polynomial-time computable.
Second, it is based on a mapping between the two vertex sets which allows to enforce a consistency over time.
This consistency assumption is useful in many temporal network applications where the vertices in both networks correspond to the same set of objects over time.
This implicitly allows to identify vertices within the two networks.
Third, vertex signatures allow for a high flexibility since they can be chosen arbitrarily in order to incorporate essential information (local or global) for the application at hand (e.g., one might use feature vectors obtained via network embedding).

\paragraph*{Graph Distance Based on Vertex Signatures.}

The following approach is due to \citet{JT09}. For a (static) graph $G=(V,E)$, a \emph{vertex signature function} $f_G \colon V \to \Q^k$ encodes arbitrary information about a vertex.
Let $\dist \colon \Q^k\times \Q^k \to \Q$ be a metric.

For two (static) graphs $G=(V,E)$ and $H=(W,F)$ with vertex signatures $f_G \colon V \to \Q^k$ and $f_H \colon W \to \Q^k$ and a given vertex mapping~$M$ between~$V$ and~$W$, we define the \emph{cost} of~$M$ as
\begin{align*}
C(G,H,M) \coloneqq{}& \sum_{\mathclap{(u, v) \in M}} \dist\big(f_G(u), f_H(v)\big)
+ \sum_{\mathclap{v\in V\setminus V_M}}\eps_G(v) 
+ \sum_{\mathclap{v\in W\setminus W_M}}\eps_H(v),
\end{align*}
where~$\eps_G(v)\in\Q$ is the (predefined) cost of ``deleting'' vertex~$v$ from~$G$ since it is not mapped by~$M$ to any vertex in the other vertex set.
The value~$\eps_G(v)$ might for example depend on the vertex signature of~$v$.
Note that ``deleting'' a vertex does not affect the signatures of other vertices.
Note also that one of the last two sums on the right-hand side above is always zero.
 
The vertex-signature-based distance between~$G$ and~$H$ is then defined as
\[
D(G,H) \coloneqq \min_{M\in\mathcal{M}(V,W)} C(G,H,M).
\]

Depending on the application, one might normalize the distance~$D$ by some appropriate factor (typically depending on~$|V|$ and~$|W|$; e.g., \citet{JT09} normalize by $\min(|V|,|W|)^{-1}$).

Throughout this work, we assume that vertex signature functions~$f_G$ are computable in polynomial time in the size of~$G$ and we assume all metrics~$\dist$ to be polynomial-time computable.
In the rest of the paper, we neglect the running times for computing the values of $f_G$ and~$\dist$ because we assume that all vertex signatures are precomputed once in polynomial time.

\paragraph*{Dynamic Time Warping Distance for Temporal Graphs.}
We transfer the concept of dynamic time warping to temporal graphs in the following way.
Let $\mathcal{G} = (V, E_1, \ldots, E_T)$ and $\mathcal{H} = (W, F_1, \ldots, F_U)$ be two temporal graphs and let
$f_{G_1},\ldots,f_{G_T} \colon V \to \Q^k$ and $f_{H_1}, \ldots,f_{H_U} \colon W \to \Q^k$ be corresponding vertex signature functions.

We then define the vertex-signature-based \emph{dynamic temporal graph warping distance} (dtgw-distance) between $\mathcal{G}$ and $\mathcal{H}$ as 
\begin{align*}
\dtgw(\mathcal{G},\mathcal{H}) \coloneqq \min_{M\in\mathcal{M}(V,W)} \; \min_{p\in\mathcal{P}_{T,U}} \sum_{(i, j) \in p} C(G_i,H_j,M).
\end{align*}

\begin{figure}
\centering
\includegraphics[scale=0.7]{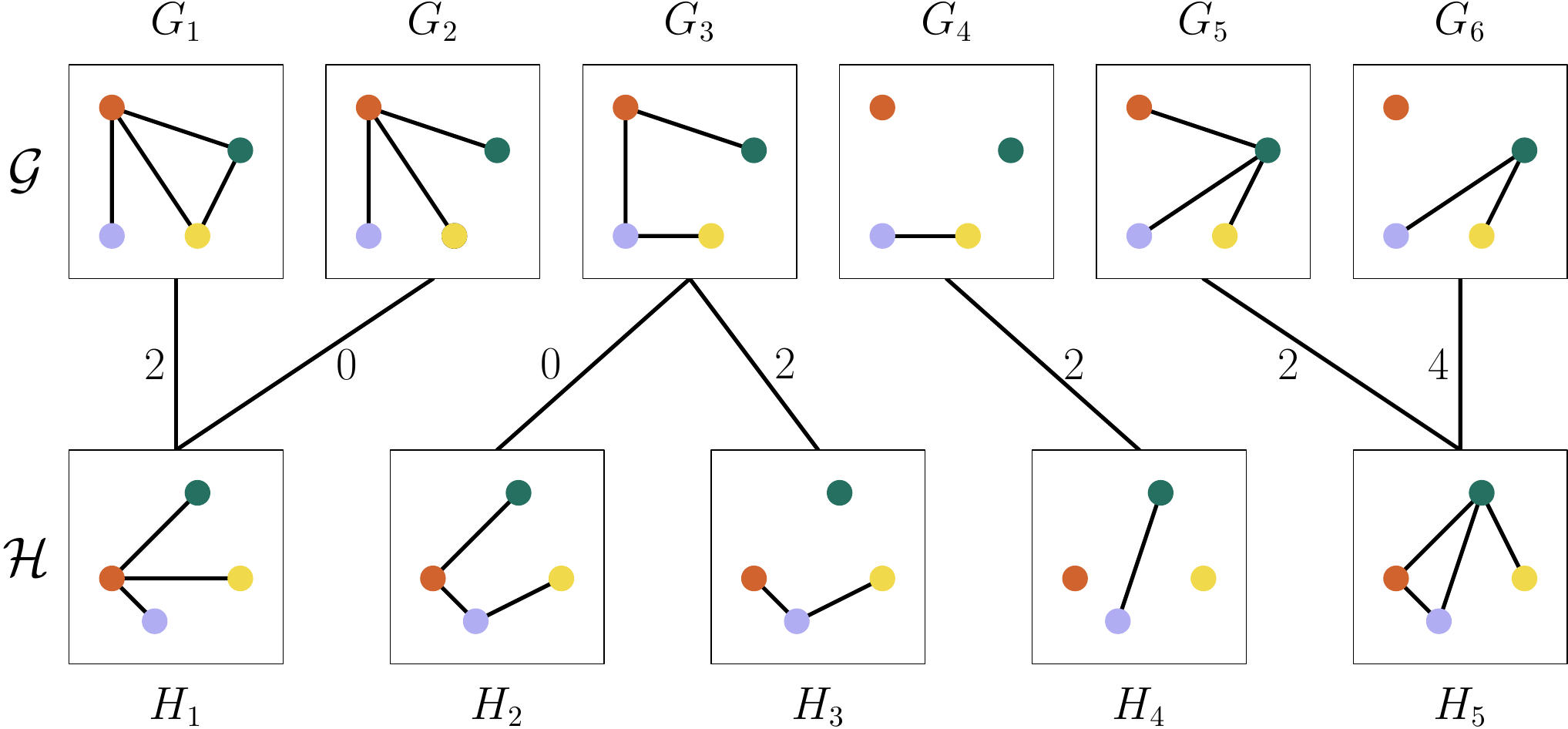}
\caption{Example of the dtgw-distance between two temporal graphs~$\Gg$ and~$\Hh$ on four vertices with lifetimes six and five. The vertex coloring indicates an optimal vertex mapping~$M$.
The connections between the boxes indicate an optimal warping path $p=\{(1,1),(2,1), (3,2), (3,3), (4,4), (5,5), (6,5)\}$.
Their labels correspond to the costs~$C(G_i,H_j,M)$ where the vertex signatures are the degrees and the metric is the absolute value of the difference.
For example, the cost of warping $G_1$ to $H_1$ is 2, as the green and yellow vertex each have degree two in $G_1$ but only degree one in $H_1$.
 The resulting dtgw-distance is $\dtgw(\Gg,\Hh) = 2+0+0+2+2+2+4=12$.}
\label{fig:DTGWexample}
\end{figure}

\noindent
\Cref{fig:DTGWexample} depicts an example illustrating the dtgw-distance of two temporal graphs.
Intuitively, the vertex mapping identifies vertices with similar behavior over time and the warping path identifies the time layers with similar vertex behavior.
Note that (for $T=U)$ if one fixes $p = \set{(1,1),(2,2),\ldots,(T,T)}$, then we get a temporal graph distance without time warping (similar to the Euclidean distance).

The following results are easily observed and play a central role for our subsequent algorithms.
\begin{observation}
  \label{obs:trivial}
  Let $\mathcal{G} = (V, E_1, \ldots, E_T)$ and $\mathcal{H} = (W, F_1, \ldots, F_U)$ be two temporal graphs with $\abs{V} \le \abs{W}\eqqcolon n$.
  \begin{compactenum}[i)]
  \item For a fixed vertex mapping~$M\in\mathcal{M}(V,W)$,
    a warping path $p\in\mathcal{P}_{T,U}$ which minimizes $\sum_{(i,j)\in p}C(G_i,H_j,M)$
    can be computed in~$\bigO(T\cdot U\cdot n)$ time.
  \item For a fixed warping path~$p\in\mathcal{P}_{T,U}$,
    a vertex mapping~$M\in\mathcal{M}(V,W)$ which minimizes $\sum_{(i,j)\in p}C(G_i,H_j,M)$
    can be computed in~$\bigO(n^2\cdot \abs{p} + n^3)$ time.
  \end{compactenum}
\end{observation}

\begin{proof}
  $i)$~For a given vertex mapping~$M$, an optimal warping path
    can be computed by a well-known dynamic program for dynamic time warping in~$\bigO(T\cdot U \cdot n)$ time~\cite{SC78}.
    Here,~$\bigO(n)$ is the time required to compute the costs~$C(G_i,H_j,M)$.
    Note that faster dynamic time warping algorithms for special cases are known \cite{ABW15,GS18,Kuszmaul19,FJR20}.

    $ii)$
    Let~$V'\coloneqq V \cup Q$, where $Q$ is a set of $|W|-|V|$ dummy vertices (that is, $|V'|=n$) with~$Q\cap V=\emptyset$.
    For every~$(u,v)\in V'\times W$, let
    $$\sigma(u,v) \coloneqq \begin{cases}\sum_{(i,j)\in p}\dist\big(f_{G_i}(u),f_{H_j}(v)\big), &u\in V\\
    \sum_{(i,j)\in p}\eps_{H_j}(v), &u\in Q\end{cases}.$$
    Then, we need to find a vertex mapping $M\in\mathcal{M}(V',W)$ that minimizes $\sum_{(u,v)\in M}\sigma(u,v)$.
    Note that~$M$ defines a bijection between~$V'$ and~$W$.
    Hence, computing~$M$ is an \textsc{Assignment Problem} instance solvable in~$\bigO(n^3)$ time~\cite[Theorem~12.2]{AMO93}.
    Computing all $\sigma(u,v)$ values can be done in~$\bigO(n^2\cdot \abs{p})$ time.
\end{proof}

Note that \Cref{obs:trivial}~i) implies that if we already know the vertex mapping up to a constant number of vertices, then $\dtgw$ can be computed in polynomial time (since we can try out all polynomially many possible vertex mappings).
Furthermore, \Cref{obs:trivial}~ii) implies that~$\dtgw$ is polynomial-time computable if the optimal temporal alignment between $\Gg$ and $\Hh$ is known beforehand.
In particular, $\dtgw$ can be computed in polynomial time if one temporal graph has a constant lifetime or a constant number of vertices since there are only polynomially many possible warping paths or polynomially many vertex mappings.

\begin{corollary}
  The dtgw-distance between two temporal graphs can be computed in polynomial time if at least one of the following applies:
  \begin{compactenum}[i)]
    \item The vertex mapping is known up to a constant number of vertices.
    \item The warping path is known.
    \item At least one of the temporal graphs has a constant lifetime or a constant number of vertices.
  \end{compactenum}
  
\end{corollary}

For given vertex signature function and metric, we refer to the decision problem of testing whether
two temporal graphs have dynamic temporal graph warping distance at most some given value~$c$ by \DTGW.

\problemdef{Dynamic Temporal Graph Warping (DTGW)}
{Two temporal graphs~$\mathcal{G}$ and $\mathcal{H}$, $c\in\Q$.}
{Is $\dtgw(\mathcal{G},\mathcal{H}) \le c$?}

\section{Computational Hardness}\label{sec:hardness}

Even though the dynamic time warping distance and the vertex-signature-based graph distance are both computable in polynomial time, 
their combined application to temporal graphs yields a distance measure that is generally \NP-hard to compute. Intuitively, this is due to the fact that the vertex mapping has to be consistent for all layers. This introduces non-trivial dependencies between the time warping and the vertex mapping
which render the problem computationally hard.
Indeed, this is not a singular case for temporal graph problems where for many cases the temporal counterparts of problems solvable in polynomial time turn \NP{}-hard; examples include the computation of matchings in graphs \cite{HHKNRS19,BBR20,MMNZZ20}, short path computations \cite{CHMZ19,FNSZ20}, or the computation of separators \cite{ZFMN20}.

\begin{theorem}\label{thm:nphard}
\DTGW is \NP{}-complete for every metric when the vertex signatures are vertex degrees.
\end{theorem}

\begin{proof}
\DTGW is clearly contained in \NP{} since for a given vertex mapping and warping path (both having polynomial size), one can check in polynomial time whether the~$\dtgw$-distance is at most~$c$ (also see \Cref{obs:trivial}).

To show \NP{}-hardness, we give a polynomial-time many-one reduction from the \NP{}-complete \threesat{} problem. 
Let $d: \Q \times \Q \to \Q$ be any metric 
and let $\phi = C_1 \wedge \ldots \wedge C_m$ be an instance of \threesat{} over the variables $x_1, \ldots, x_n$.
Each clause $C_j$ is then a disjunction of three literals $C_j =: \ell_j^1 \lor \ell_j^2 \lor \ell_j^3$ 
and there is a function $\nu: [m] \times [3] \to [n]$ such that 
$\ell_j^i \in \set{x_{\nu(j, i)}, \overline{x_{\nu(j, i)}}}$ holds for all $\ell_j^i$.
Without loss of generality we assume $m > 8$.

The idea is to represent each literal by a vertex which can be mapped to either $\top$ (\true) or $\bot$ (\false).
We then build, for each clause, a \emph{clause box} gadget consisting of three consecutive layers.
The choice of a warping path will then, for each clause, implicitly select one of its literals and the costs
caused by each clause box will attain their minimum value if and only if that particular literal is mapped to $\top$.

Now a detailed description of the reduction follows.
Let $D$ and $D'$ be two copies of the graph $\coprod_{i=1}^{22m} K_2$ (consisting of $22m$ disjoint edges), where for each vertex $v \in V(D)$ we denote its copy in $V(D')$ by $v'$.
We construct two temporal graphs $\mathcal{G}$ and $\mathcal{H}$.
Their vertex sets each contain $2n + 47m + 8$ vertices as follows.
\begin{align*}
V(\mathcal{G}) &\coloneqq \set{x_i, \overline{x_i} ; i\in [n]} \cup 
\set{C_j^1, C_j^2, C_j^3 ; j \in [m]}
\cup \set{X_i, Y_i ; i \in [4]} \cup V(D),\\
V(\mathcal{H}) &\coloneqq \set{\top_i, \bot_i ; i\in [n]} \cup 
\set{C_j'^1, C_j'^2, C_j'^3 ; j \in [m]}
\cup \set{X'_i, Y'_i ; i \in [4]} \cup V(D').
\end{align*}
Both temporal graphs have~$2n+26m$ layers defined as follows.
For each $i \in [n]$, we set
\begin{align*}
  E(G_{2i-1}) &\coloneqq \{\{x_i, \overline{x_i}\}\}, & E(H_{2i-1}) &\coloneqq \{\{\top_i, \bot_i\}\},\\
  E(G_{2i}) &\coloneqq E(D), & E(H_{2i}) &\coloneqq E(D') \,.
\end{align*}
For $j \in [m]$, we set
\begin{align*}
E(G_{2n + 4j - 3}) &\coloneqq \set{\set{X_i, Y_i} ; i \in [4]}, \\
E(G_{2n + 4j - 2}) &\coloneqq \set{\{C_j^i, \ell_j^i\}; i \in [3]}, \\
E(G_{2n + 4j - 1}) &\coloneqq \set{\set{X_i, Y_i} ; i \in [4]}, \\
  E(G_{2n + 4j}) &\coloneqq E(D),
\end{align*}
and
\begin{align*}
E(H_{2n + 4j - 3}) \coloneqq{}& \set{\set{C_j'^1, \top_{\nu(j, 1)}}, \set{\top_{\nu(j,2)}, \bot_{\nu(j,2)}}, \set{\top_{\nu(j,3)}, \bot_{\nu(j,3)}}, \set{C_j'^2, C_j'^3} },\\
E(H_{2n + 4j - 2}) \coloneqq{}& \set{\set{C_j'^2, \top_{\nu(j,2)}}, \set{\top_{\nu(j,1)}, \bot_{\nu(j,1)}}, \set{\top_{\nu(j,3)}, \bot_{\nu(j,3)}}, \set{C_j'^1, C_j'^3}} \\
& \cup \set{\set{X'_i, Y'_i}; i \in [4]},\\
E(H_{2n + 4j - 1}) \coloneqq{}& \set{\set{C_j'^3, \top_{\nu(j,3)}}, \set{\top_{\nu(j,1)}, \bot_{\nu(j,1)}}, \set{\top_{\nu(j,2)}, \bot_{\nu(j,2)}}, \set{C_j'^1, C_j'^2} },\\
E(H_{2n + 4j}) \coloneqq{}& E(D') \,.
\end{align*}
Finally, for $j \in [22m]$, we set
\begin{align*}
E(G_{2n+4m + j}) &:= \set{\set{X_k, Y_k} ; k \in [4]}, &
E(H_{2n+4m + j}) &:= \set{\set{X'_k, Y'_k} ; k \in [4]} \,.
\end{align*}

We call the layers containing $\abs{E(D)}$ edges \emph{separation layers}.
Furthermore, for each $j \in [m]$ we say that the layers $2n+4j - 3$, $2n+4j-2$, and $2n+4j-1$ form the \emph{clause block} corresponding to $C_j$ (see \cref{fig:clause_box} for an example).

\begin{figure}
\includegraphics[width=\textwidth]{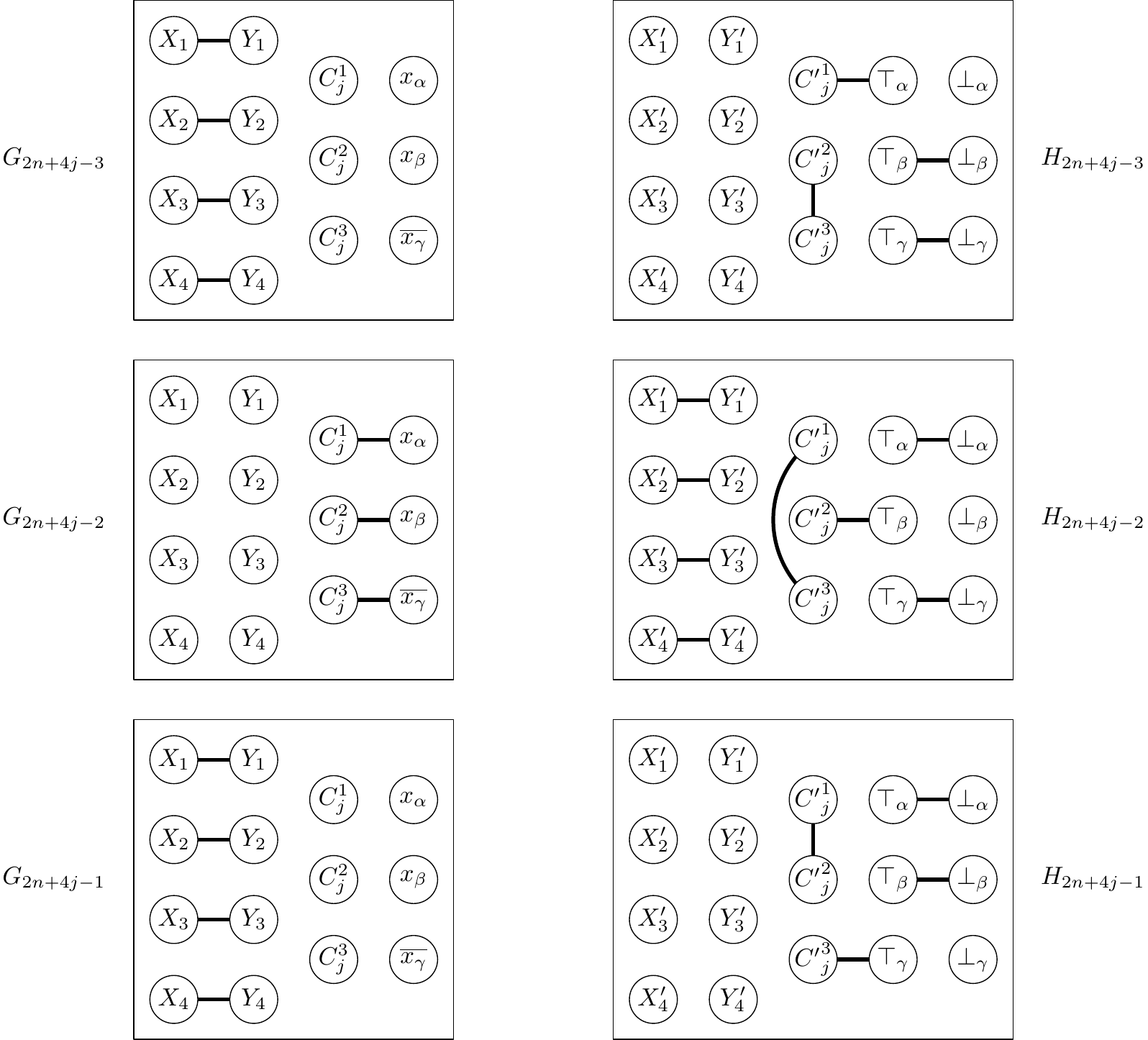}
\caption{Clause block for $C_j = x_\alpha \lor x_\beta \lor \overline{x_\gamma}$. Only relevant vertices are shown in each layer.}
\label{fig:clause_box}
\end{figure}

Let $c\coloneqq  42m \cdot d(0, 1)$. We claim that $\dtgw(\Gg, \Hh) \leq c$ if and only if $\phi$ has a satisfying assignment.

\medskip

``$\Leftarrow$'':
Given a satisfying assignment $\beta \colon \{x_1,\ldots,x_n\}\rightarrow \{\true,\false\}$ of~$\phi$, we define the following vertex mapping
\begin{align*}
M :={}& \set{(x_i , \top_i), (\overline{x_i}, \bot_i) ; \beta(x_i)=\true} \\
& \cup \set{(x_i, \bot_i), (\overline{x_i}, \top_i); \beta(x_i)=\false}\\
& \cup \set{(C_j^i, {C'}_j^i) ;  i \in [3], j \in [m]} \\
& \cup \set{(X_i, X'_i), (Y_i, Y'_i) ; i \in [4]} \\
& \cup \set{(v, v') ; v \in V(D)}.
\end{align*}

\begin{figure}
\centering
\begin{subfigure}{0.32\textwidth}
\includegraphics[scale=0.9,page=1]{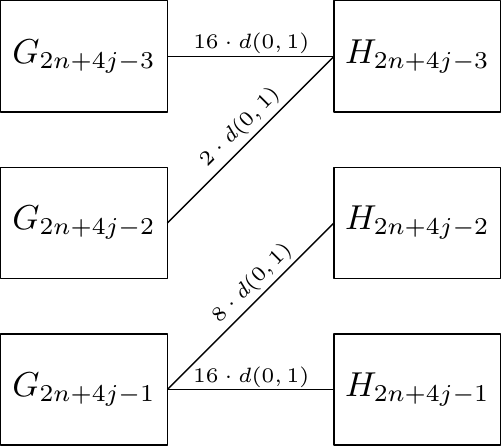}\caption{}
\label{fig:warping1}
\end{subfigure}
\begin{subfigure}{0.32\textwidth}
\includegraphics[scale=0.9,page=2]{figures/NPhardness_clausebox_warping.pdf}\caption{}
\label{fig:warping2}
\end{subfigure}
\begin{subfigure}{0.32\textwidth}
\includegraphics[scale=0.9,page=3]{figures/NPhardness_clausebox_warping.pdf}\caption{}
\label{fig:warping3}
\end{subfigure}
\caption{The three possible warpings between layers of a clause block. 
  Each edge is labeled with the minimal cost it causes under the assumption 
that the set $\set{X_i, Y_i; i \in [4]}$ is mapped to $\set{X'_i, Y'_i; i \in [4]}$.}
\label{fig:warping}
\end{figure}

To construct a warping path, we begin by defining, for each $j \in [m]$, the following three sub-paths (see also \cref{fig:warping}):
\begin{align*}
\pi_j^1 &:= \big\{ (2n+4j-2, \; 2n+4j-3),   (2n+4j-1,\; 2n+4j-2)\big\},\\
\pi_j^2 &:= \big\{ (2n+4j-2, \;2n+4j-2) \big\},\\
\pi_j^3 &:= \big\{  (2n+4j-3, \; 2n+4j-2), (2n+4j-2, \; 2n+4j-1) \big\}.
\end{align*}
For each clause $C_j = \ell_j^1 \lor \ell_j^2 \lor \ell_j^3$, pick $k_j \in [3]$ such that $\ell_j^{k_j}$ is true.
We then build the warping path~$p$ as the union of all $\pi_j^{k_j}$, using the trivial warping path for all remaining layers:
\begin{align*}
p := \set*{(i, i) ; i \in [2n+4m+22m] \setminus \set{2n+4j-2 ; j \in [m]}}
\cup \bigcup_{j\in[m]} 
\pi_j^{k_j}.
\end{align*}

It is then not difficult to calculate that each clause block adds cost of exactly $42 \cdot d(0, 1)$ and there are no other costs.
Thus $\dtgw(\Gg, \Hh) \leq 42m \cdot d(0, 1) = c$.

\medskip

``$\Rightarrow$'':
Now suppose that $\dtgw(\Gg, \Hh) \leq c$ and let $(M, p)$ be a pair of vertex mapping~$M$ and warping path~$p$ with cost $\sum_{(i,j) \in p} C(G_i, H_j, M) = \dtgw(\Gg, \Hh)$.
Note that any non-separation layer contains at most eight edges.
So if $p$ warps any separation layer to any non-separation layer, then the resulting cost would be at least $(44m - 16) \cdot d(0, 1) > c$.
Thus, we may assume that every separation layer~$i$ of~$\mathcal{G}$ is  only warped to layer~$i$ of~$\mathcal{H}$ and vice versa.
Since the last $22m$ layers of each temporal graph are all identical and $M$ and~$p$ are chosen to have minimal cost, we can conclude that 
\begin{align*}
p \supset \set*{(i, i) ; i \in [2n+4m+22m] \setminus \set{2n+4j-2 ; j \in [m]}}.
\end{align*}

If~$M$ maps some vertex from $\set{X_k, Y_k ; k \in [4]}$ to some vertex that is not in~$\set{X'_k, Y'_k; k \in [4]}$, then the $22m$ layers $2n + 4m + 1,\ldots,2n + 4m + 22m$ each would cause cost of at least $2 \cdot d(0, 1)$, thus exceeding~$c$ in total.
Hence, $M$~has to contain a bijection from $\set{X_k, Y_k; k \in [4]}$ to $\set{X'_k, Y'_k; k \in [4]}$.

Now, consider the clause block corresponding to $C_j = \ell_j^1 \lor \ell_j^2 \lor \ell_j^3$.
From the arguments above, it follows that $G_{2n+4j-3}$ and $G_{2n+4j-1}$ are warped to $H_{2n+4j-3}$ and $H_{2n+4j-1}$, respectively.
This already costs $32 \cdot d(0, 1)$.
We distinguish three cases (corresponding to $\pi_j^1$ through $\pi_j^3$ above):
\begin{enumerate}[(1)]
\item\label{case1} $G_{2n+4j-2}$ is warped to $H_{2n+4j-3}$.
This causes costs of at least $2 \cdot d(0, 1)$.
Then, $H_{2n+4j-2}$ must be warped to $G_{2n+4j-1}$ or $p$ would not have minimal cost.
Thus, there are additional costs of at least $8 \cdot d(0, 1)$.
This is the situation illustrated in \cref{fig:warping1}.
\item\label{case2} $G_{2n+4j-2}$ is warped to $H_{2n+4j-2}$. This causes costs of at least $10\cdot d(0, 1)$.
This is the situation illustrated in \cref{fig:warping2}.
\item\label{case3} $G_{2n+4j-2}$ is warped to $H_{2n+4j-1}$.
This case is symmetric to (\ref{case1}) and also causes costs of at least $10 \cdot d(0, 1)$.
This is the situation illustrated in \cref{fig:warping3}.
\end{enumerate}
In summary, the costs contributed by each clause block are at least $42 \cdot d(0, 1)$.
Hence, to meet the bound of $c$, all layers outside of clause blocks must not cause any additional cost.
For each $i \in [n]$, since $G_{2i-1}$ is warped to $H_{2i-1}$, this implies that either $\{(\top_i,x_i),(\bot_i,\overline{x_i})\}\subset M$ or $\{(\top_i,\overline{x_i}),(\bot_i,x_i)\}\subset M$.

Furthermore, for each $j \in [m]$, the clause block corresponding to $C_j$ must have cost of exactly $42 \cdot d(0, 1)$.
If we are in Case~(\ref{case1}) as above, then this is only possible if~$M$ maps each degree-1 vertex of~$G_{2n+4j-2}$ to some degree-1 vertex of~$H_{2n+4j-3}$.
Thus, $\left(\ell_j^2, \top_{\nu(j,2)}\right) \in M$.
Otherwise, if we are in Case~(\ref{case2}) respectively Case~(\ref{case3}),
then analogous arguments yield that
$\left(\ell_j^1, \top_{\nu(j, 1)} \right) \in M$ respectively
$\left(\ell_j^3, \top_{\nu(j, 3)} \right) \in M$.
Hence, in any case there is some $i \in [3]$ for which $\left(\ell_j^i, \top_{\nu(j, i)}\right) \in M$.

Consequently, 
\begin{align*}
\beta(x_i) \coloneqq \begin{cases}
\true, & \text{if $(x_i, \top_i) \in M$}\\
\false, & \text{if $(\overline{x_i}, \top_i) \in M$}
\end{cases}
\end{align*}
is a satisfying assignment for $\phi$.
\end{proof}

Let us take a closer look at the reduction in the proof of \cref{thm:nphard}.
Note that the corresponding optimal warping path is always close to the diagonal (that is, $|i-j|\le 1$ holds for every pair~$(i,j)$). Hence, it lies within the so-called Sakoe-Chiba band~\cite{SC78} of width~one.
Moreover, the maximum degree in each layer is one.
Finally, the number of vertices and the number of layers of both temporal graphs and the target cost~$c$ are all upper-bounded linearly in the size of the \threesat{} formula, which allows to
conclude a running time lower bound based on the \emph{Exponential Time Hypothesis}\footnote{%
	The Exponential Time Hypothesis asserts that \threesat cannot be solved in subexponential time,
	that is, that there is no $2^{o(n)}\cdot\text{poly}(m)$-time algorithm, where $n$~is the number of variables and~$m$ is the number of clauses of the input formula.
}~\cite{IP01} (together with the \emph{Sparsification Lemma}~\cite{IPZ01}).
These observations are summarized in the following corollary.
(Recall that $V$, $W$ are the vertex sets and $T$, $U$ are the lifetimes of $\Gg$, $\Hh$.)

\begin{corollary}
  \DTGW is \NP{}-complete for every metric and vertex degrees as vertex signatures even when the maximum degree of each layer is one and the warping path is restricted to the Sakoe-Chiba band of width one.
  Moreover, even this very restricted variant cannot be solved in $2^{o(|V|+|W|+T+U+c)}\cdot \mathrm{poly}(|\Gg|+|\Hh|)$ time unless the Exponential Time Hypothesis fails.
\end{corollary}

Furthermore, if the dtgw-distance is normalized (e.g., divided by the number of vertices),
then we obtain \NP-hardness for a constant value of~$c$ (by the reduction in the proof of \Cref{thm:nphard}).

\begin{corollary}\label[corollary]{cor:constc}
  \DTGW with normalized vertex-signature-based distance is \NP-complete for a constant value of~$c$.
\end{corollary}

Due to the proven worst-case hardness of \DTGW, there is little hope to solve the problem efficiently in general. Nevertheless, in the next section we provide some methods to cope with this intractability.

\section{Algorithms}\label{sec:algorithms}
In this section we first point out polynomial-time solvable special cases. 
Then, we develop a mathematical programming formulation as well as a heuristic approach to approximate the $\dtgw$-distance in practice.

\subsection{Exact Polynomial-time Algorithms for Special Cases}
Our first algorithmic result is to show that one can determine in polynomial time
whether two temporal graphs with the same number of vertices have $\dtgw$-distance zero.
This basic case occurs when checking for duplicates within a data set.
In contrast, determining whether two (static) graphs have graph edit distance zero is not known to be polynomial-time solvable (as this is equivalent to the famous \textsc{Graph Isomorphism} problem).

\begin{theorem}\label{thm:c=0}
  Let~$\mathcal{G}=(V,E_1,\ldots,E_T)$ and~$\mathcal{H}=(W,F_1,\ldots,F_U)$ be two temporal graphs with $|V| = |W| =n$.
  For all vertex signatures and all metrics, deciding whether $\dtgw(\Gg,\Hh)=0$ holds is possible in $\bigO(n^2\cdot (T+U) + n^3)$ time.
\end{theorem}

\begin{proof}
  We will show that for distance zero, an optimal warping path can easily be determined.
  Polynomial-time solvability then follows from \cref{obs:trivial}.
  
  Let~$\mathcal{G}=(V,E_1,\ldots,E_T)$ and~$\mathcal{H}=(W,F_1,\ldots,F_U)$ be two temporal graphs with~$V=:\{v_1,\ldots,v_n\}$ and~$W=:\{w_1,\ldots,w_n\}$. 
  For each~$i\in[T]$, we define the \emph{$i$\textsuperscript{th} layer signature} of~$\mathcal{G}$ as~$f(G_i)\coloneqq \big(f_{G_i}(v_1),\ldots,f_{G_i}(v_n)\big)$  
  (analogously, $f(H_j)\coloneqq\big(f_{H_j}(w_1),\ldots,f_{H_j}(w_n)\big)$ for~$j\in[U]$).
  Assuming $\dtgw(\mathcal{G},\mathcal{H})=0$, it follows that there exists a vertex mapping~$M\subseteq V\times W$ and a warping path~$p\in\mathcal{P}_{T,U}$ such that
  $$\sum_{(u, v) \in M} \dist\big(f_{G_i}(u), f_{H_j}(v)\big)=0$$
  holds for every $(i,j)\in p$.
  Since~$\dist$ is a metric, this implies that~$f_{G_i}(u)=f_{H_j}(v)$ holds for every~$(u,v)\in M$.
  That is, $f(H_j)$ is a permutation (determined by~$M$) of~$f(G_i)$.
  Let $1\le i_1 < i_2 \ldots < i_q < T$ be the indices such that
  \[f(G_i) \neq f(G_{i+1}) \iff i \in \set{i_k ; k \in [q]} \]
  and let~$1\le j_1 < j_2 < \ldots < j_r < U$ be the indices such that
  \[ f(H_j) \neq f(H_{j+1}) \iff j \in \set{j_k; k \in [r]} \,.\]
  Clearly, if~$f(G_i)\neq f(G_{i'})$ and layer~$i$ is warped to layer~$j$ and layer~$i'$ is warped to layer~$j'$, then $f(H_j)\neq f(H_{j'})$ since otherwise the cost will not be zero.
  By the definition of a warping path, it follows that the layers~$1,\ldots,i_1$ of~$\mathcal{G}$ can only be warped to layers~$1,\ldots,j_1$ of~$\mathcal{H}$ and the layers~$i_1+1,\ldots,i_2$ of~$\mathcal{G}$ can only be warped to layers~$j_1+1,\ldots,j_2$ of~$\mathcal{H}$ and so on.
  Note that this is only possible if~$q=r$. If this is the case, then we can
  assume that the warping path~$p$ has the following form:
  \begin{align*}
    p = \bigl\{&(1,1), (1, 2), \ldots, (1,j_1), (2, j_1), \ldots(i_1,j_1),\\
               &(i_1+1,j_1+1),\ldots,(i_1+1,j_2),\ldots(i_2,j_2),\\ 
               & \ldots, \\
               &(i_q+1,j_q+1),\ldots,(i_q+1,U),\ldots,(T,U) \bigr\}.
  \end{align*}
  By \cref{obs:trivial}, we can now check whether there exists a vertex mapping that yields distance zero for the warping path~$p$
  in $\bigO(n^2\cdot (T+U) + n^3)$ time. 
  Computing~$p$ can be done in~$\bigO(n(T + U))$ time.
\end{proof}

We remark that if the vertex signatures and the metric satisfy the property that every pair of different vertex signatures has distance at least~$\delta$ for some constant~$\delta>0$, then \DTGW parameterized by the resulting cost~$c$ is in \XP{}.
For example, this is the case when the vertex signatures contain only integers and~$\dist$ is
any $\ell^p$-norm (for $p\ge 1$). Then, every pair of different signatures has distance at least~$\delta =1$.
The idea of the algorithm is to ``guess'' the tuples of a warping path which cause non-zero cost (at most $c / \delta$ many) and to check whether it is possible to complete the warping path without further costs.
The latter can be done in polynomial time using similar arguments as for the case~$c=0$ (\Cref{thm:c=0}).

\begin{corollary}
  \DTGW is in \XP with respect to~$c$ if the vertex signatures have a constant minimum pairwise distance~$\delta > 0$.
\end{corollary}

In contrast, if the dtgw-distance is normalized, then the differences between vertex signatures can be arbitrarily small, in which case \DTGW is \NP-hard for constant~$c$ (\Cref{cor:constc}).

To overcome this hardness, in the following, we consider special cases based on parameters regarding the warping path length.
We assume that the lifetimes of the inputs differ by at most a constant, that is, $T = U + t$ for some $t \ge 0$ (which might often be the case in practice).
Note that, by definition, every warping path of order~$T\times U$ has length at least~$T$.
We define the parameter~$\lambda$ to be the difference between the warping path length and the lower bound~$T$, that is, we consider only order-$(T\times U)$ warping paths of length at most~$T+\lambda$ (in practice, overly long warping paths might be considered unnatural).
We prove that \DTGW is in \XP with respect to the combined parameter~$(\lambda,t)$.

\begin{theorem}\label{prop:XPwarplength}
  For all vertex signatures and all metrics, \DTGW is solvable in
  $$\bigO\left((T+\lambda)^\lambda\cdot T^{\lambda+t}\left(n^2\cdot(T+\lambda)+n^3\right)\right)$$
  time if $n = \max(|V|,|W|)$, $T=U+t$, and the warping paths have length at most~$T+\lambda$.
\end{theorem}

\begin{proof}
  Let~$\mathcal{G}=(V,E_1,\ldots,E_T)$ and~$\mathcal{H}=(W,F_1,\ldots,F_T)$ be two temporal graphs and let~$p=\set{p_1=(i_1,j_1),\ldots,p_L=(i_L,j_L)}\in\mathcal{P}_{T,U}$ be a warping path.
  The warping path~$p$ contains~$L-1$ \emph{steps} $p_{\ell+1}-p_\ell=(i_{\ell+1}-i_\ell,j_{\ell+1}-j_\ell)\in\{(1,0),(0,1),(1,1)\}$ for $1\le\ell<L$.
  We call a step~$\ell$ \emph{horizontal} if $p_{\ell+1}-p_\ell=(1,0)$, and we call it \emph{vertical} if~$p_{\ell+1}-p_\ell=(0,1)$, and otherwise we call it \emph{diagonal}.
  Let~$\nu\le U-1$ denote the number of vertical steps in~$p$.
  Then, $p$~contains also $\nu + t$ horizontal and $U-1-\nu$ diagonal steps, that is,
  $L-1 = \nu + \nu + t + U - \nu -1$, which implies that $\nu= L - t - U$.
  Clearly, there are $\binom{L-1}{\nu}$ possible positions for the vertical steps.
  For each of these possible choices, there are again $\binom{L-1-\nu}{\nu+t}$ possible positions for horizontal steps (the remaining steps are diagonal).
  Therefore, the overall number of warping paths of length at most~$T+\lambda$ is
  \[\sum_{l=0}^\lambda \binom{T+l-1}{l}\binom{T-1}{l+t} \in \bigO\left((T+\lambda)^\lambda\cdot T^{\lambda+t}\right).\]
  For each of these possible warping paths, we can compute $\dtgw(\Gg,\Hh)$ in~$\bigO(\max(|V|,|W|)^2\cdot(T+\lambda)+\max(|V|,|W|)^3)$ time by \cref{obs:trivial}.
\end{proof}

Note that \Cref{prop:XPwarplength} implies polynomial-time solvability of \DTGW if~$t$ and~$\lambda$ are constants.
For unbounded~$t$, however, we conjecture that \DTGW is \NP-hard even if the warping paths are restricted to have length~$\max(T,U)$, which is the minimum possible length (that is, $\lambda=0$).

\subsection{Quadratic Programming}\label{sec:qp}

We give a formalization of \DTGW as a quadratic minimization problem with linear constraints~(QP).
This is \NP-hard to solve in general but can be used to solve small instances exactly with state-of-the-art QP-solvers such as Gurobi\footnote{\url{www.gurobi.com}}.

Let $\Gg = (V,E_1,\ldots,E_T)$ and $\Hh=(W,F_1,\ldots,F_U)$ be two temporal graphs.
Denote the vertices in~$V$ by $u_1,\ldots,u_{|V|}$ and the vertices in~$W$ by $v_1,\ldots,v_{|W|}$.
To model ``vertex deletion'', we add two artificial vertices $u_{|V|+1}, v_{|W|+1}$.
We use the following variables:
\begin{compactitem}
  \item For every $(i,j)\in[|V|+1]\times[|W|+1]$, we have a \emph{vertex mapping variable}~$m_{i,j}\in\{0,1\}$,
  where $m_{i,j}=1$ if and only if vertex~$u_i$ is mapped to vertex~$v_j$.
  \item For every $(s,t)\in[T]\times[U]$, we have a \emph{warping variable}~$w_{s,t}\in\{0,1\}$,
    where $w_{s,t}=1$ if and only if~$G_s$ is warped to~$H_t$.
\end{compactitem}

Moreover, for every $(s,t,i,j)\in[T]\times[U]\times[|V|+1]\times[|W|+1]$, let
\[
d_{s,t,i,j} \coloneqq \begin{cases}
  \dist\big(f_{G_s}(v_i), f_{H_t}(w_j)\big), & i\in[|V|],\; j\in[|W|]\\
  \eps_{G_s}(v_i), & i\in[|V|],\; j=|W|+1\\
  \eps_{H_t}(w_j), & i=|V|+1,\; j\in[|W|]\\
  0, & i=|V|+1,\; j=|W|+1
\end{cases}
\]
denote the cost of matching vertex $i$ in layer $s$ to vertex $j$ in layer $t$.

Then, computing~$\dtgw(\Gg,\Hh)$ is the following quadratic\footnote{
It is also possible to convert our formulation into a linear problem by introducing further variables and constraints for replacing the product $w_{s,t}\cdot m_{i,j}$ in the objective function.
However, in our experiments we found that the quadratic formulation can be solved faster.
} minimization problem.
\begin{subequations}\label{eqn:ILP}
\begin{alignat}{9}
& \text{minimize} \quad && \mathrlap{\sum_{s\in[T]} \; \sum_{t\in[U]} \; \sum_{i\in[|V|+1]} \; \sum_{j\in[|W|+1]} d_{s,t,i,j} \cdot w_{s,t}\cdot m_{i,j} } \tag{\ref*{eqn:ILP}}  \\
&\text{subject to} \quad &&& 
\sum_{j\in[|W|+1]} m_{i,j} &= 1 &\forall i \in [|V|]\label{a} \\
&&&& \sum_{i\in[|V|+1]} m_{i,j} &= 1 &\forall j \in [|W|]\label{b}\\
&&&& w_{1,1} &= 1 \label{c}\\
&&&& w_{s,t} &\leq \mathrlap{w_{s+1,t+1} + w_{s,t+1} + w_{s+1,t}}  &\hspace{4cm}\forall (s,t) \in [T-1]\times[U-1]\label{d}\\
&&&& w_{T,t} &\leq w_{T,t+1} &\forall t \in [U-1]\label{e}\\
&&&& w_{s,U} &\leq w_{s+1,U} &\forall s \in[T-1]\label{f}
\end{alignat}
\end{subequations}

The constraints \ref{a} and \ref{b} ensure that the vertex mapping variables define a correct vertex mapping, that is, every vertex is mapped to exactly one other vertex (or is deleted).
Constraints \ref{c} to \ref{f} ensure that the warping variables define a valid warping path.
Here, the constraints~\ref{d} to \ref{f} imply that if the warping path contains a pair~$(s,t)$, then it also contains at least one of the pairs~$(s+1,t)$, $(s,t+1)$, or $(s+1,t+1)$ (since the objective is minimized, any solution will actually select only one of these pairs).

The number of variables is in $\bigO(|V|\cdot|W| + T \cdot U)$ and the number of constraints is in~$\bigO(|V|+|W|+T\cdot U)$.

\subsection{Heuristic Approaches}\label{sec:heuristic}
In this section, we present a heuristic to compute the dtgw-distance,
which typically yields good (not necessarily optimal) solutions in practice.

The approach is to simply start with an arbitrary initial vertex mapping (or warping path) and to compute an optimal warping path (vertex mapping) based on \cref{obs:trivial} in polynomial time.
This process is then repeated by alternating between optimal warping path and optimal vertex mapping computation until the solution converges to a local minimum (or some other criterion is reached).

Note that it is a convenient feature of our heuristic to be able to stop the process after any number of iterations
to obtain some approximate solution (a so-called \emph{anytime algorithm}).
It is further possible to incorporate prior knowledge, for example, by fixing the mapping for some vertices.
Note also that convergence is guaranteed since we decrease the objective in each alternation and the search space is finite.
We propose several initialization options.

\paragraph*{Initial Warping Path.}
A first idea for initialization is to choose a shortest warping path (that is, of length $\max(T,U)$).
Note that for $T \neq U$ several such paths exists.
Without further knowledge about the instances, choosing a path within the Sakoe-Chiba band of small width is a reasonable default.
This initialization is very simple and only requires $\bigO(T+U)$ time.

Another idea is to compute a warping path using $D(G_i,H_j)$ as a cost for warping layer~$i$ to layer~$j$.
This is of course an optimistic estimate since it allows to use a different vertex mapping for each pair of layers.
Then, a vertex mapping can be computed by \cref{obs:trivial}.
This initialization takes $\bigO(T \cdot U \cdot n^3)$ time where $n := \max(\abs{V}, \abs{W})$.

\paragraph*{Initial Vertex Mapping.} The idea is to compute a vertex mapping by solving an \textsc{Assignment Problem} instance for approximate costs.
Let~$\sigma(u,v)$ be some approximate cost for mapping vertex~$u\in V$ to~$v\in W$.
For example, one could use the following estimations
\begin{align*}
  \sigma^*(u,v) &\coloneqq \sum_{i\in[T]}\sum_{j\in[U]}\dist\big(f_{G_i}(u),f_{H_j}(v)\big),\\
  \sigma_{\text{opt}}(u,v) &\coloneqq \sum_{i\in[T]}\min_{j\in[U]}\dist\big(f_{G_i}(u),f_{H_j}(v)\big).
\end{align*}
The first option~$\sigma^*$ estimates the cost of mapping~$u$ to~$v$ over all possible warpings between any two layers (this is usually more than any warping path will incur).
The definition of~$\sigma_{\text{opt}}$ only considers for each layer of the first temporal graph the minimal cost over all layers of the other temporal graph (this estimate might be too low).
Both of these options require $\bigO(T \cdot U \cdot \abs{V} \cdot \abs{W})$ time.

Based on the estimated costs one computes a vertex mapping by solving an \textsc{Assignment Problem} instance
and then computes an optimal warping path for this vertex mapping based on~\cref{obs:trivial}.

\medskip

The running time of one iteration (that is, computing a vertex mapping and an optimal warping path) is $\bigO\left((T + U)\cdot n^2 + n^3 + T \cdot U \cdot n\right)$.
While the number of iterations might depend on the choice of initialization, in our experiments
the heuristic always converged after very few iterations.
Regarding the solution quality, while it is possible to construct adversarial examples where the heuristic performs poorly, our experiments in \Cref{sec:benchmark} indicate that it performs well in practice.

\section{Experiments}\label{sec:experiments}

We conducted several experiments\footnote{Source code available at \url{www.akt.tu-berlin.de/menue/software}.} to demonstrate the merit of our dtgw-distance in applications and to evaluate the performance of the alternating minimization heuristic~(\AM) we described in \Cref{sec:heuristic}.
For computations we used a 4.0~GHz i7-6700K processor (single-threaded).

\subsection{Data Sets}\label{sec:data}
We used three data sets from the SocioPatterns project~\cite{GB18}.
Each of these consists of two temporal social networks, both recorded simultaneously with the same individuals as vertices.
The first network is a face-to-face contact network whereas the second one is a co-presence network where edges represent spatial proximity.
All six networks have a temporal resolution of 20 seconds. For our experiments, only the first day of each network was used and vertices without any edges were discarded. 
The three different data sets were recorded at a primary school (“LyonSchool”, 237 vertices, 1700 layers), a scientific conference (“SFHH”, 403 vertices, 2300 layers) and a workplace (“InVS15”, 180 vertices, 2100 layers).

\subsection{Comparison of Heuristic and Exact Solutions}
\label{sec:benchmark}
\newcommand{\AMI}[1]{\textsc{AM}$_{#1}$\xspace}
\newcommand{\AMSS}{\AMI{\sigma^*}}
\newcommand{\AMSO}{\AMI{\sigma_{\text{opt}}}}
\newcommand{\AMSW}{\AMI{\text{swp}}}
\newcommand{\AMOW}{\AMI{\text{owp}}}

We compared the solutions of our \AM heuristic under different initialization schemes against the optimal solutions obtained from the QP formulation.

Due to long running times of the QP-solver, we were restricted to very small temporal networks. 
We randomly selected~10 children from class~1A of the primary school face-to-face network and extracted~225 consecutive layers (during a high contact period) which we split into 15 temporal subnetworks with~15 consecutive layers each.
We used vertex degrees as signatures (with absolute value metric) and computed all pairwise dtgw-distances between these 15 networks with the following algorithms:
\begin{compactitem}
\itemsep0em
\item QP: exact QP-solver (Gurobi 8.0.1),
\item \AMSS: \AM with $\sigma^*$ initialization,
\item \AMSO: \AM with $\sigma_{\text{opt}}$ initialization,
\item \AMSW: \AM with shortest warping path initialization,
\item \AMOW: \AM with optimistic warping path initialization.
\end{compactitem}

We implemented the \AM heuristic in Python, using a C++ implementation\footnote{This code is available as a Python module at \href{https://github.com/src-d/lapjv}{github.com/src-d/lapjv}.
} of the Jonker-Volgenant algorithm \cite{Jonker1987} to solve the \textsc{Assignment Problem}.

\Cref{fig:primary-school-heuristics_cdf} shows for each initialization variant the estimated cumulative distribution function (ecdf) of the error percentage
$
\varepsilon = 100\cdot ( d_{\AM} - d_{\text{QP}})/d_{\text{QP}},
$ 
where $d_{\AM}$ is the approximate dtgw-distance obtained by an \AM heuristic and $d_{\text{QP}}$ is the exact dtgw-distance obtained by the QP-solver.
A point $(\varepsilon, P)$ on an ecdf-curve of an \AM heuristic means that the error percentage of \AM is at most $\varepsilon$ with estimated probability $P$. 

All \AM variants found the correct solution for a majority of samples~($P_0 > 0.5$). 
The average error percentages are rather small and vary between $3.0$ by \AMOW and $5.5$ by \AMSO. 
The \AMOW heuristic performed best, having $P_0 \approx 0.71$ and maximum error percentage $\max \approx 36.4$. 
These findings indicate that for small instances the approximations of the four heuristics are close to the optimal solution on average but may fail considerably in some cases with up to a maximum error of~$63.6\%$.
We remark that based on our experimental experience the relative error becomes smaller on larger instances.
 
Regarding running times, the \AM heuristic took less than 0.01 seconds per instance, usually converging after at most three iterations (independently of the chosen initialization).
In comparison, the QP was slower by a factor of more than 10\,000, requiring 8 minutes on average (median 2 minutes) with some instances approaching 2 hours.

\begin{figure}
\begin{minipage}{0.5\textwidth}
\includegraphics[width=\textwidth]{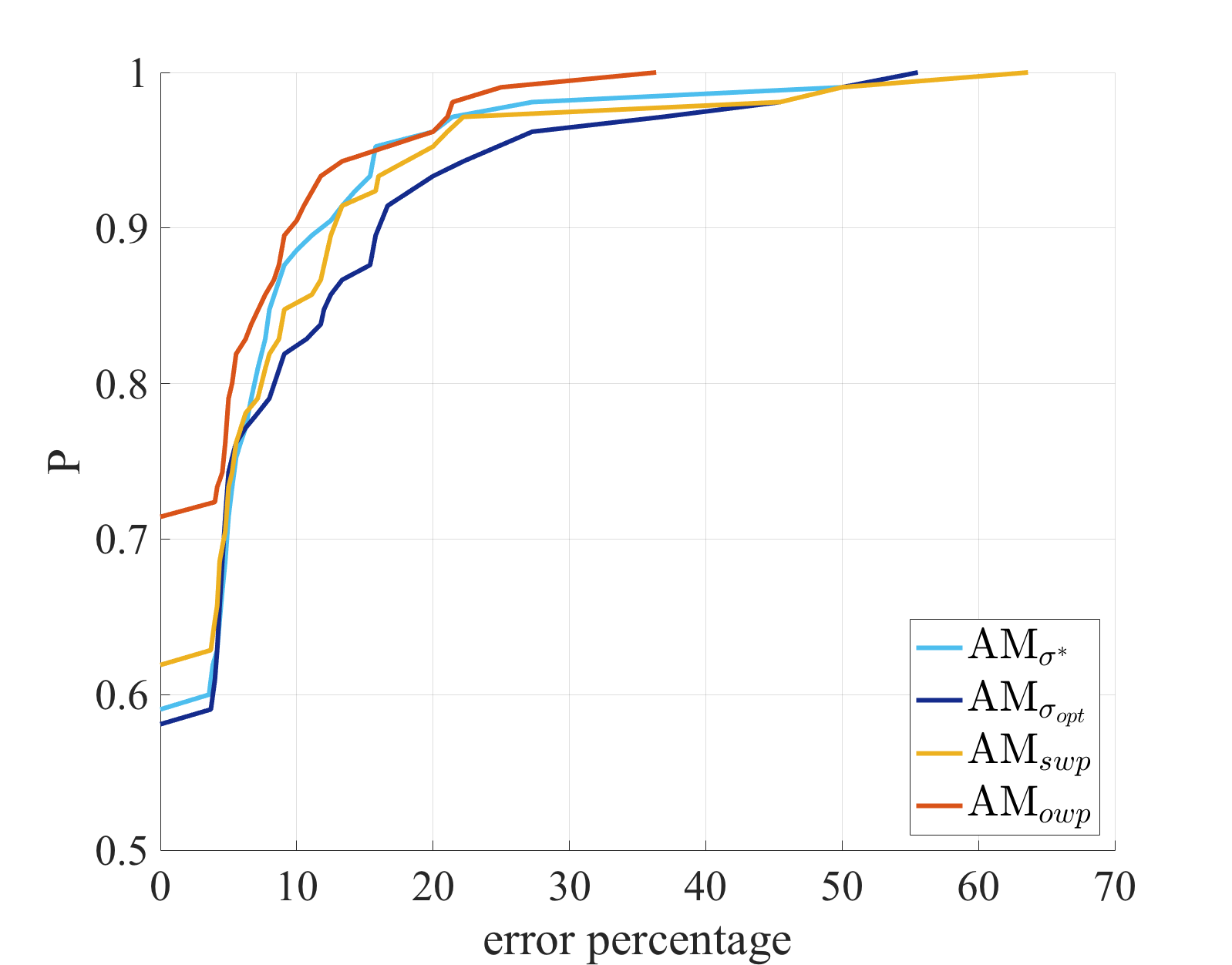}
\end{minipage}
\hfill
\begin{minipage}{0.45\textwidth}
\small
\begin{tabular}{lcccc}
\toprule
heuristic & avg & std & $P_0$ & max\\
\midrule
\AMSS & 4.5 &  9.1 & 0.59 & 63.6 \\
\AMSO & 5.5 & 10.5 & 0.58 & 55.6 \\
\AMSW & 4.8 & 10.0 & 0.61 & 63.6 \\
\AMOW & 3.0 &  6.2 & 0.71 & 36.4 \\
\bottomrule
\end{tabular}
\end{minipage}
\caption{The plot (left) shows the estimated cumulative distribution functions of the four \AM variants.
The table (right) presents the average error percentage~(avg), the standard deviation~(std), the fraction of optimally solved instances~($P_0$), and the maximum error percentage~(max) of every \AM variant.}
\label{fig:primary-school-heuristics_cdf}
\end{figure}

\subsection{Sensitivity of \DTGW to Noise}\label{sec:cluster}

The goal of this experiment was to assess how sensitive the dtgw-distance is to noise, that is,
how well can original data be reconstructed from noisy data.
We compared our dtgw-distance approach to the following two baseline methods.

\begin{itemize}

\item Non-consistent: Instead of using one consistent vertex mapping for all layers, one can allow a different mapping for each pair of layers.
Note that the resulting distance can be computed in $\bigO(T \cdot U \cdot n^3)$ time, thus being faster than an exact computation of the dtgw-distance but much slower than a single iteration of the \AM{} heuristic.

\item Non-temporal: A naive approach is to ignore the time information and solely compute an optimal vertex mapping between the underlying graphs.
This requires $\bigO(n^3)$ time plus the (usually linear) time to build the underlying graphs.
\end{itemize}

We used the primary school face-to-face network from which we extracted five reference temporal networks representing the contacts between children of the same grade,
each containing 45--50 vertices and 3100 layers.

For each of the five reference networks, we generated nine noisy copies as follows: 
\begin{enumerate}[(i)]
\item  For every~$i\in[T]$, $E_i$ is deleted with probability $p\in\{0.1,0.2,0.3\}$, and if not,
  then each edge~$e\in E_i$ is deleted with probability $p$.
\item  For every~$i\in[T]$, each edge~$e\in E_i$ was rewired with probability $p \in \{0.1, 0.2, 0.3\}$. 
\item  Each edge of the underlying graph was rewired with probability $p \in \{0.1, 0.2, 0.3\}$. 
\end{enumerate}

\noindent Rewiring an edge~$e=\{u,v\}\in E_i$ of a temporal graph is defined as
randomly picking a tuple $(e'=\{u',v'\},t)\in \bigcup_{s=1}^T(E_s \times\{s\})$
and then replacing $e$ in~$E_i$ by $\{u,v'\}$ and $e'$ in~$E_t$ by $\set{u',v}$.\footnote{
Since edges are undirected, it is understood that the choice of which vertex to call $u$ respectively $v$ is to be made randomly (the same holds for $u'$ and $v'$).}
Rewiring of underlying edges is done analogously (see \citet{HS12} for details).

We used the \AM heuristic to approximate the pairwise dtgw-distances (using degrees as vertex signatures) between all reference and noisy temporal networks.
In all of these instances, shortest warping path initialization (which is fastest) was used since preliminary tests showed that the other initializations produce very similar results.

\Cref{fig:primary-school-noise} shows the dendogram obtained by hierarchical clustering using complete linkage of the approximated pairwise dtgw-distances.
Both, the dtgw-distance and the non-consistent baseline were able to partition the instances into five clusters, each of which consists of a reference network and its nine noisy copies.
Hence, they successfully recovered the original reference networks from noise.
However, the clusters produced by the dtgw-distance are more compact than the ones of the non-consistent baseline.
In contrast, the non-temporal baseline was not able to separate the graphs of different grades.

In all instances, the heuristic converged within at most six iterations,
taking less than 15 seconds.
In comparison, the non-consistent baseline required 4 minutes on average for each instance,
while the non-temporal baseline was the fastest (below 1 second).

\begin{figure}[t]
\includegraphics[width=0.49\textwidth]{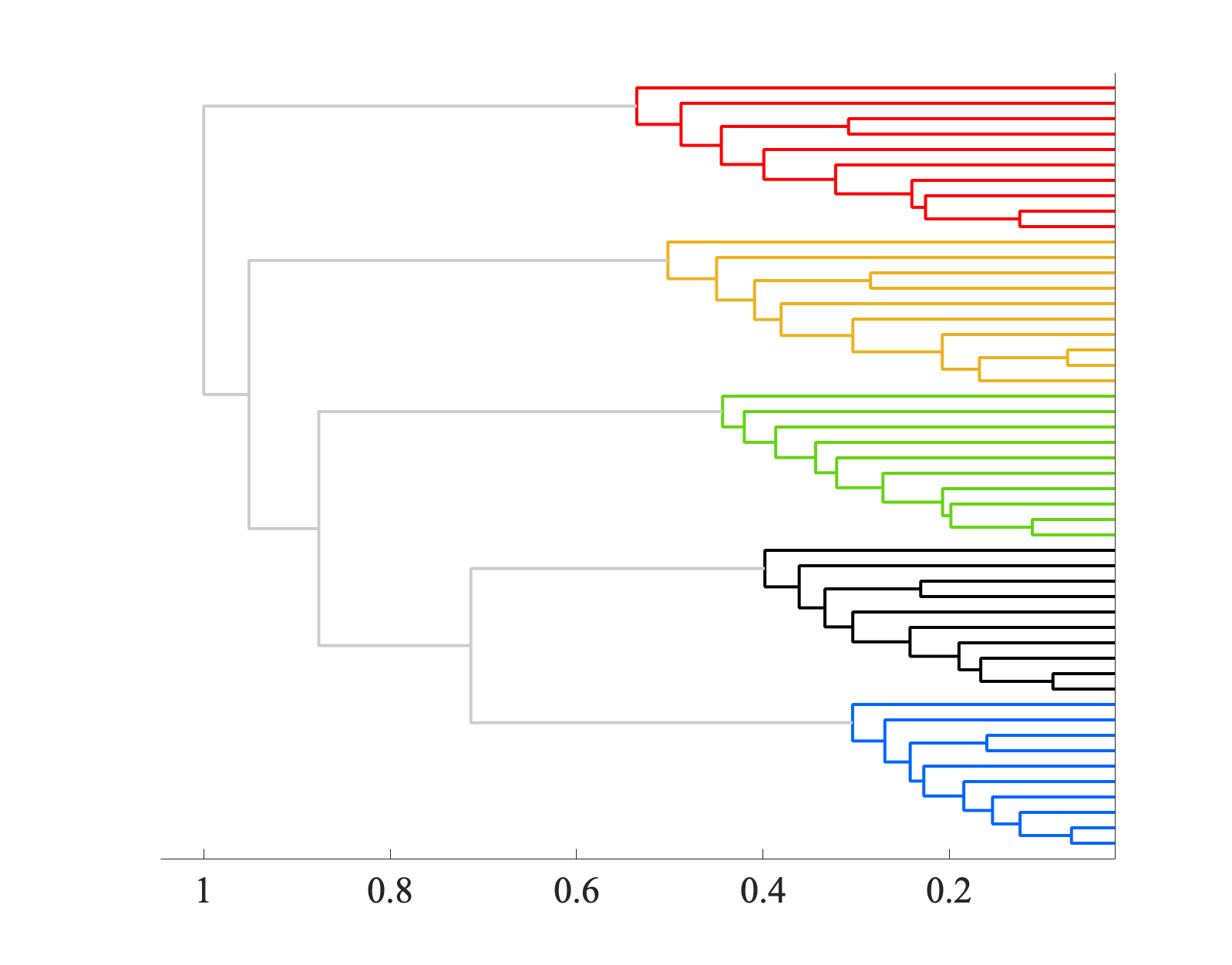}
\hfill
\includegraphics[width=0.49\textwidth]{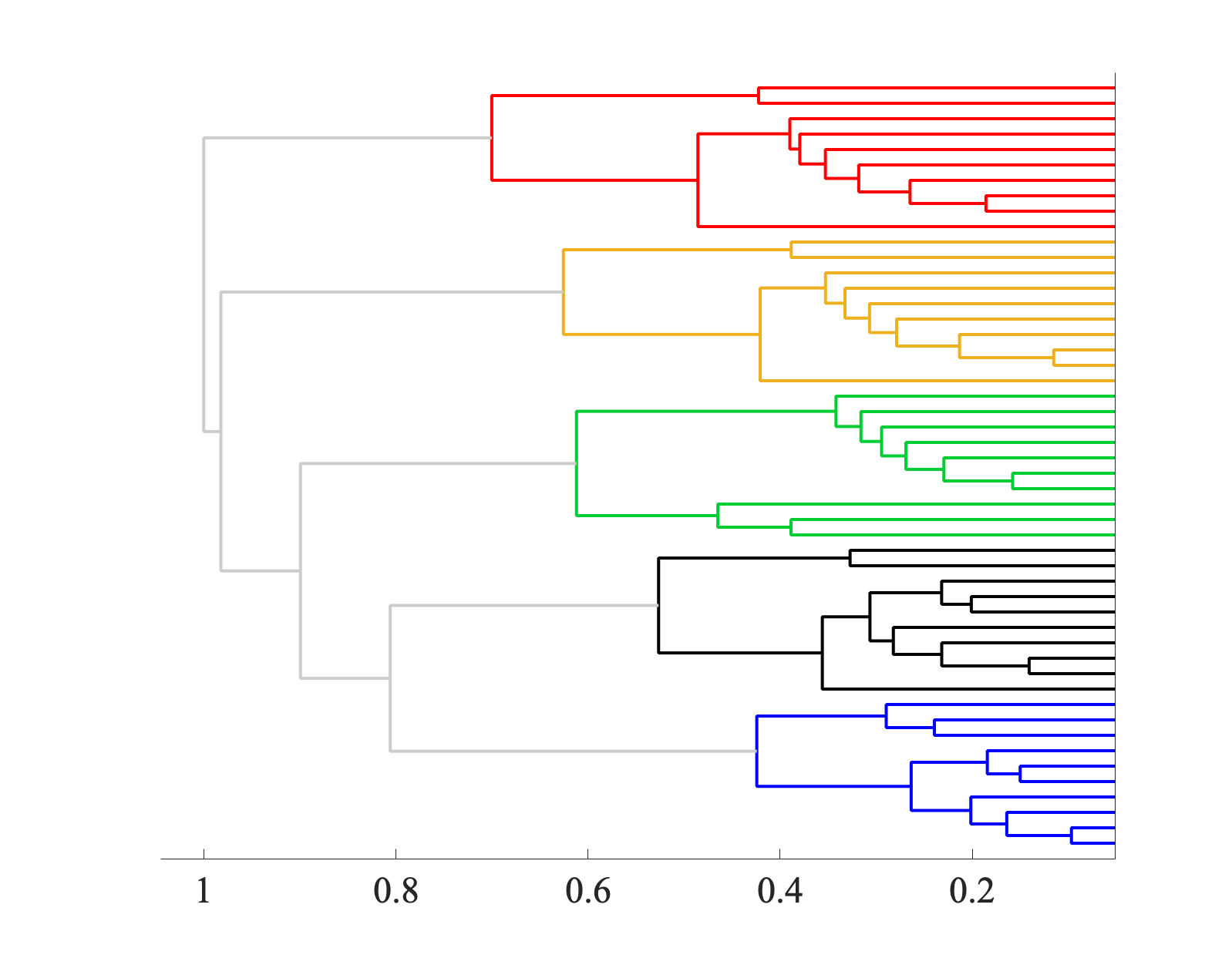}\\
\centering
\includegraphics[width=0.49\textwidth]{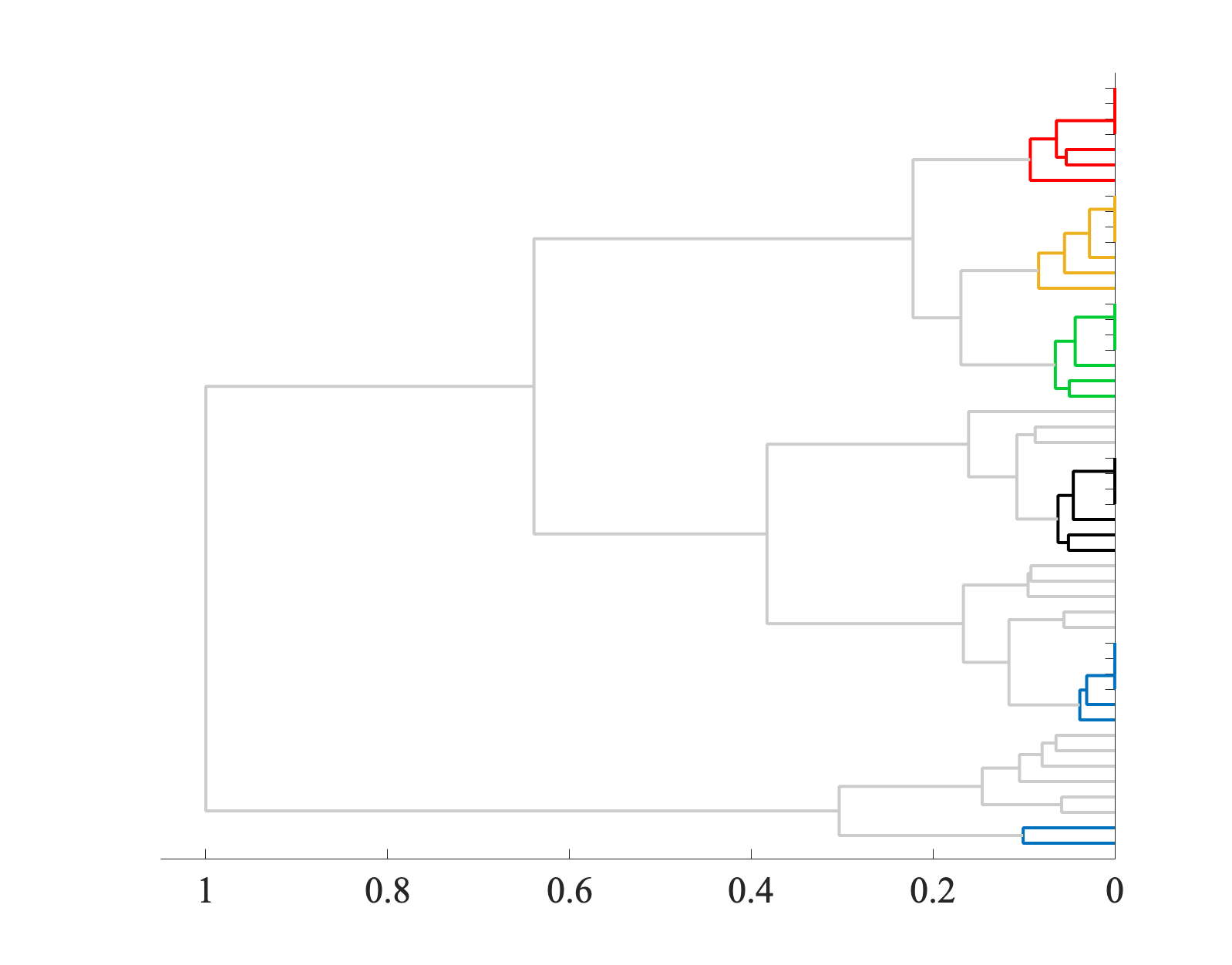}
\caption{Sensitivity of dtgw to noise. 
Shown are the dendograms obtained by agglomerative clustering using complete linkage.
Different colors represent different grades, the light gray edges connect elements of different grades.
The top left dendogram shows the result obtained by the dtgw-distance.
The top right dendogram shows the result obtained by the non-consistent benchmark.
The bottom dendogram shows the result obtained by using the non-temporal benchmark.
}
\label{fig:primary-school-noise}
\end{figure}

\subsection{De-Anonymization}
Besides measuring a distance between temporal graphs, the dtgw-distance additionally provides a mapping between the vertex sets which implicitly allows to identify vertices.
This allows the de-anonymization \cite{NS09} of temporal social networks.
Since the data sets used contain the original mapping between the vertext sets,
we can employ this as an easy benchmark for the accuracy of the dtgw-distance.

We used the \AM heuristic (with shortest warping path initialization) to compute the dtgw-distance (with degrees as vertex signatures\footnote{We also tested other signatures such as size of the connected component or betweenness centrality. However, the performance was (slightly) worse.}) on the three data sets mentioned in \Cref{sec:data}. We counted how many vertices were correctly re-identified (that is, mapped to their copies) in the resulting vertex mapping.
We compared our results to the following alternative algorithms found in the literature:
\newcommand{\dynamagna}{DynaMAGNA++}
\begin{itemize}
  \item \dynamagna{}~\cite{VCM17}: A search-based evolutionary algorithm computing a vertex mapping that maximizes edge conservation and node conservation over time.
  \item Temporal Network Embedding~\cite{ZLLGHW18}: Hawkes Process Based Temporal Network Embedding (HTNE) computes a low-dimensional embedding of the vertices of a temporal network. From this, we computed a vertex mapping minimizing the Euclidean distances between the vertex feature vectors.
  \item Fixed dtgw: Note that our dtgw-distance allows to fix the warping path beforehand (\Cref{obs:trivial}~ii)).
	Since in our case each pair of temporal graphs was recorded using synchronized clocks, it is natural to use a fixed warping path that aligns layer~$i$ of the first graph with layer~$i$ of the second graph.
\end{itemize}

To simulate a situation in which the temporal graphs represent processes which do not run synchronously in time, we created two modified versions of each of the data sets.
In the first one, called “shifted”, all events of the first graph were delayed by three minutes.
In the second version, called “randomized”, each layer of each of the graphs was randomly and independently replaced by $X$ layers where $X \in \{1, 2, \dots\}$ is a random variable with $\mathbb{P}(X\geq x) = x^{-3}$.
Since we pretend that the nature of these modifications is unknown to the tested algorithms, dtgw with fixed warping path is not applicable to these variants.

For \dynamagna{}, we used a population of size~15\,000 and a maximum of 10\,000 generations.
With HTNE, we computed 128-dimensional vertex embeddings using a batch of size 10\,000, a learning rate of 0.1, a history length of 2, and 5 negative samples.
Unlike dtgw, both methods utilized all four processor cores.

The results and running times are listed in \Cref{tab:vertex_identification}.
Most notably, the re-identification rate of HTNE was poor on all data sets, suggesting that these embeddings are ill-suited for comparing vertices taken from different networks.
Furthermore, all methods failed to re-identify any significant number of vertices on the primary school data set.
This might be explained by the fact that the co-presence network is very different from the face-to-face contacts due to a low spatial resolution (as was also noted by \citet{GB18}).

The overall performance was much better on the other two data sets, especially on the conference data where up to 90\% of participants could be re-identified 
whereas on the workplace data set the best result was 51\%.
Unsurprisingly, fixing the correct layer alignment on the unmodified graphs sped up the dtgw computation significantly while also yielding slightly better results.
On these instances, dtgw performed comparably to \dynamagna{}, being better in one case and worse in the other, although requiring much less computational effort.
In contrast, on the shifted and randomized data sets dtgw always achieved the best results (notably the re-identification performance using dtgw did not decrease on shifted data).

In all cases the \AM heuristic converged after at most six iterations and \dynamagna{} converged within 2\,000 generations.

\begin{table}
  \caption{Percentages (rounded) of vertices that were re-identified by the tested methods. Also average running times (in seconds) over the three versions of each data set are given (fixed dtgw is not applicable for shifted and randomized data sets).}
	\centering
	\renewcommand{\tabcolsep}{7pt}
	\begin{tabular}{ll rrrr}
	\toprule
	&	data set 				& dtgw 	& fixed dtgw & DynaMAGNA++	& HTNE \\
	\midrule
	\multirow{4}{*}{\rotatebox[origin=c]{90}{school}} 
		& original 	       & 2\% & 1\% & 1\% & 1\% \\
		& shifted	       & 1\% & -- & 0\% & 1\% \\
		& randomized	   & 1\% & --   & 1\% & 0\% \\
		& average running time  & 95\,s & 65\,s  & 15\,070\,s  & 250\,s \\
	\midrule
	\multirow{4}{*}{\rotatebox[origin=c]{90}{conference}} 
		& original         & 86\% & 90\% & 80\% & 0\% \\
		& shifted          & 86\% & --   & 27\% & 0\% \\
		& randomized           & 65\% & --   & 30\% & 1\% \\
		& average running time  & 200\,s  & 125\,s  & 20\,320\,s  & 90\,s \\
	\midrule
	\multirow{4}{*}{\rotatebox[origin=c]{90}{workplace}} 
		& original 	       & 38\%  & 43\%  & 51\% & 1\% \\
		& shifted 	       & 38\%  & --    & 19\% & 0\% \\
		& randomized       & 10\%  & --    &  8\% & 1\% \\
		& average running time  & 45\,s   & 20\,s   & 1\,600\,s & 50\,s \\
	\bottomrule
	\end{tabular}
	\label{tab:vertex_identification}
\end{table}

\section{Conclusion}\label{sec:conclusion}

We introduced a new proximity measure for comparing temporal graphs by transferring dynamic time warping from time series to temporal graphs.
This yields a challenging computational problem for which we proposed exact algorithms and a heuristic approach to solve it.
While exact solutions can only be computed for very small instances,
we empirically showed that our heuristic runs fast in practice and yields good approximations of optimal solutions.
In our experiments, it was also capable of de-anonymizing social networks.

Our work opens several directions for future research.
We believe that the dtgw-distance is a promising tool for example in biology and chemistry.
Processes like epidemic disease spreading or chemical reactions can naturally be viewed as temporal graphs where the vertices represent individuals or (macro) molecules 
(unfortunately we could not test this, as there is still a lack of openly available temporal molecular data \cite{VCM17}).
Since the exact time scales of these processes often vary, the ability of dynamic time warping to compensate for such differences would be especially helpful in this context.
One might also use the dtgw-distance to understand the learning process of neural networks.
The training phases of neural networks yield temporal networks which can be analyzed to gain insight into how different conditions influence the learning process.
Another potential application is analyzing team sports data via temporal graphs to reveal similar strategies
or roles of individual players.
Depending on the application domain, there is a wide range of possibilities to test the performance when using different vertex signatures or even other graph distances.

Besides experimenting with various application domains and further definition variants, already the proven computational worst-case hardness of \DTGW may trigger further algorithmic research.
A concrete open question is whether \DTGW is in \XP (or even fixed-parameter tractable) with respect to  $\lambda$, when the warping path length is restricted to be at most $\max(T,U)+\lambda$ where $T$, $U$ are the respective lifetimes.
It is also interesting to further study the influence of graph-specific parameters in the spirit of a multivariate complexity analysis \cite{Nie10,FJR13}.

\bibliography{bib}

\end{document}